\documentclass{llncs}
\usepackage{graphicx}

\usepackage{amsmath,amssymb,amsthm}
\usepackage{enumitem}

\usepackage[bibliography=common]{apxproof}

\renewcommand{\phi}{\varphi}

\def\props{\textit{Prop}}
\def\O{\mathbf O}
\def\F{\mathbf F}
\def\M{\mathcal M}
\def\lang{\mathcal{L}}
\def\atoms{\Omega}
\def\modelspl{\Vdash_{\textnormal{CL}}}
\def\impliespl{\models_{\textnormal{CL}}}
\def\vdashpl{\vdash_{\textnormal{CL}}}

\def\MDL{\textnormal{MDL}}
\def\langmdl{\lang_{\MDL}}
\def\modelsmdl{\Vdash_{\MDL}}
\def\impliesmdl{\models_{\MDL}}

\def\tdl{\textnormal{TDL}}
\def\langtiny{\lang_{\tdl}}
\def\mtiny{\M^{\tdl}}
\def\atomstiny{\atoms^{t}}

\def\ldl{\textnormal{LDL}}
\def\langl{\lang_{\ldl}}
\def\modelsl{\Vdash_{\ldl}}
\def\impliesl{\models_{\ldl}}
\def\vdashl{\vdash_{\ldl}}

\def\ldlb{\textnormal{LDLB}}
\def\langb{\lang_{\ldlb}}
\def\modelsk{\Vdash_{\ldlb}}
\def\impliesk{\models_{\ldlb}}
\def\vdashk{\vdash_{\ldlb}}
\def\atomsb{\atoms^{B}}
\def\agents{\mathit{Ag}}
\def\rdda{\textnormal{RDDA}}

\newcommand{\statevocab}{\Sigma}
\newcommand{\decatoms}{\mathcal D}
\newcommand{\trueatoms}{\mathcal T}
\newcommand{\falseatoms}{\mathcal F}
\newcommand{\countednorms}{N}
\newcommand{\normweight}{v}
\newcommand{\actagent}{a}
\newcommand{\dncoptuple}{(\statevocab,\trueatoms,\falseatoms,\decatoms,\decconstraint,\countednorms,\normweight)}
\newcommand{\dbar}{\overline{\decatoms}}
\newcommand{\siga}{\statevocab^{\actagent}}
\newcommand{\choice}[2]{\mathsf{Choice}(#1,#2)}
\newcommand{\objdesc}[2]{\mathsf{Obj}(#1,#2)}
\newcommand{\subjdesc}[2]{\mathsf{Subj}(#1,#2)}
\newcommand{\objcomp}[2]{\mathsf{Mod}^{\mathrm{obj}}(#1,#2)}
\newcommand{\subjcomp}[2]{\mathsf{Mod}^{\mathrm{sub}}(#1,#2)}
\newcommand{\violformula}[1]{\chi_{#1}}
\newcommand{\violcount}[2]{\mathsf{viol}(#1,#2)}
\newcommand{\objviol}[1]{\mathsf{viol}^{\mathrm{obj}}(#1)}
\newcommand{\subjviol}[1]{\mathsf{viol}^{\mathrm{sub}}(#1)}
\newcommand{\objopt}[1]{\mathsf{viol}_{\min}^{\mathrm{obj}}(#1)}
\newcommand{\subjopt}[1]{\mathsf{viol}_{\min}^{\mathrm{sub}}(#1)}
\newcommand{\objopta}[1]{\mathsf{optdec}^{\mathrm{obj}}(#1)}
\newcommand{\subjopta}[1]{\mathsf{optdec}^{\mathrm{sub}}(#1)}
\newcommand{\objval}[1]{\nu_{#1}}
\newcommand{\decconstraint}{\chi}
\newcommand{\COMMA}{,\allowbreak}
\newcommand{\decisions}[1]{\mathsf{Dec}(#1)}

\theoremstyle{plain}
\newtheoremrep{thm}{Theorem}
\newtheoremrep{lem}[thm]{Lemma}
\newtheoremrep{prop}[thm]{Proposition}
\newtheoremrep{defn}[thm]{Definition}
\newtheoremrep{ex}[thm]{Example}

\title{A Simple Doxastic Deontic Logic for Norm-Guided Decision Making}

\author{Thorsten~Engesser \and Agata~Ciabattoni}
\authorrunning{T. Engesser and A. Ciabattoni}
\institute{TU Wien\\
  \email{thorsten@logic.at, agata@logic.at}}

\begin{document}
\maketitle

\begin{abstract}
  Making decisions despite conflicting norms and incomplete or unreliable information is a
  fundamental challenge for autonomous systems. We introduce a simple doxastic deontic logic for this
  setting: a classically reducible fragment of Chellas' Minimal Deontic Logic, extended with explicit
  conditional norms and combined with multi-agent KD45, so that norms can depend on agents' beliefs
  about both facts and norms. On this logic we define the Doxastic Norm Compliance Optimization
  Problem, where an agent chooses a decision minimizing weighted norm violations. We distinguish
  subjective optimization (relative to the agent's beliefs) from objective optimization (relative to
  the actual facts). We give conditions under which (i) the two coincide and (ii) optimal
  decision-making can be reduced to weighted partial MaxSAT in polynomial time.
\end{abstract}

\section{Introduction}
AI systems are increasingly deployed in domains governed by legal, organizational, or ethical
norms, which specify obligatory and forbidden actions. Autonomous vehicles, assistive robots, or
decision-support systems must therefore select actions that comply with applicable regulations and
policies. However, these systems may face conflicting norms and partial or unreliable beliefs about
their environment, making it difficult to determine which norms apply.

Norm-guided decision making therefore raises two key challenges. First, full compliance with norms
may be impossible. In such cases, an intelligent agent should select the action that minimizes the
severity of norm violations (the lesser evil). Second, many norms are conditional, so their
applicability may depend on the agent’s beliefs about the world. The agent must therefore determine
which norms apply before selecting an action.

Existing approaches address different aspects of these challenges but do not provide a unified
framework for doxastic normative reasoning that is both expressive and suitable for automated
decision making. Deontic logics remain the principal tool to formalize and reason about
norms~\cite{Hand13,Hand21}. Among these, preference-based deontic logics~\cite{A84,xavier_survey}
provide elegant treatments of conditional norms and contrary-to-duty (CTD)
obligations\footnote{Obligations that come into force when other obligations are violated.} but
struggle with factual detachment, i.e., inferring a concrete obligation from a conditional norm
once the relevant facts are established, e.g.~\cite{DBLP:conf/deon/CiabattoniU025}. Input/Output
logics~\cite{DBLP:journals/jphil/MakinsonT00}, instead, naturally support factual detachment of
conditional norms, and in their constrained variants~\cite{DBLP:journals/jphil/MakinsonT01}, can
accommodate certain conflicts and CTD reasoning, but not higher-order norms. Moreover, these
formalisms do not work well in the presence of dilemmas (where an action is both obligatory and
prohibited) or belief-dependent norms. Logic-programming approaches
(e.g.~\cite{DBLP:journals/cacm/SergotSKKHC86,DBLP:journals/jphil/GovernatoriORS13,DBLP:journals/jphil/KowalskiS18}),
including deontic extensions of Answer Set Programming (via modal
operators~\cite{DBLP:journals/jancl/AlferesGL13,DBLP:conf/icail/GiordanoMD13} or atomic
norms~\cite{DBLP:conf/jelia/CabalarCT23}), offer practical implementations but have not been
combined with doxastic reasoning.

Motivated by this gap, we develop a logical framework for norm-guided decision making under
imperfect information. Rather than aiming for a comprehensive theory of deontic reasoning, our
objective is to identify a fragment that captures the essential mechanisms needed for practical
decision support while remaining simple and easily extensible. Our starting point is a restricted
fragment of Chellas' Minimal Deontic Logic~\cite{DBLP:books/cu/Chellas80}, in which obligations and
prohibitions apply only to atomic normative contents. We show that reasoning in this fragment
reduces to ordinary classical propositional reasoning. We then enrich it with explicitly
represented conditional norms and factual detachment, obtaining a Lightweight Deontic Logic (LDL)
that captures key aspects of conditional normative reasoning, including higher-order norms for
expressing obligations that govern when other obligations apply.

To capture the role of beliefs in norm applicability, we combine LDL with multi-agent doxastic
logic~\cite{Hintikka1962,DBLP:journals/ai/HalpernM92}, allowing explicit norms to depend on agents'
beliefs, while agents themselves may reason about both factual information and their beliefs about
norms. Building on this logical framework, we introduce the Doxastic Norm Compliance Optimization
Problem (DNCOP), in which an acting agent must choose among alternative decisions while minimizing
violations of weighted norms, adopting a sort of utilitarian
perspective~\cite{mill1863utilitarianism} that aggregates norm violations into a single
quantitative measure of harm. The resulting framework supports encoding CTDs and exceptions through
optimization of weighted norms, while separately distinguishing optimization based on the agent's
subjective beliefs from optimization based on the objective state of the world. We identify natural
locality and truthfulness conditions under which these two perspectives coincide; under these
conditions, computing optimal decisions reduces in polynomial time to weighted partial
MaxSAT~\cite{DBLP:conf/sat/AnsoteguiBL09}, allowing existing optimization technology to be used
directly.

The resulting framework combines logical transparency with computational efficiency: norms remain
explicitly represented, beliefs affect their applicability, and decisions are obtained through a
well-defined optimization problem. Its modular design also allows integration with existing
epistemic logics~\cite{Hintikka1962,DBLP:journals/ai/HalpernM92,DBLP:conf/ecai/LakemeyerL12}, thus
supporting more expressive normative reasoning in multi-agent systems.

\section{Preliminaries: Chellas' Minimal Deontic Logic}
Chellas’ Minimal Deontic Logic (MDL)~\cite{DBLP:books/cu/Chellas80} is often presented as a
response to the deontic paradoxes that have beset the first formalisation of deontic logic—Standard
Deontic Logic~\cite{von1967deontic} (i.e., the normal modal logic KD)—and as a system that
tolerates deontic conflicts. MDL is a very weak system whose language $\langmdl$ is defined as
$\phi := p \in \props \mid \neg \phi \mid \phi \land \phi \mid O\phi$, where (for the purposes of
this paper) $\props$ is a finite set of (classical) atoms.

$O\phi$ means “$\phi$ is obligatory”. We define $F\phi := O\neg\phi$ (“$\phi$ is forbidden”) and
$P\phi := \neg F\phi$ (“$\phi$ is permitted”). Boolean connectives and constants are as usual.

Formulas from $\langmdl$ are evaluated over \emph{neighborhood models} $\M = (W,N,V)$, where $W$ is
a set of \emph{worlds}, $N : W \to 2^{2^W}$ is a \emph{neighborhood function}, and $V : W \to
  2^{\props}$ assigns to each world a propositional \emph{valuation}. As frame properties, Chellas'
semantics assumes that $\emptyset \notin N(w)$ for all $w \in W$, and that $N(w)$ is closed under
supersets. That is, whenever $U \in N(w)$ and $U \subseteq X \subseteq W$, then $X \in N(w)$. Truth
in a world $w$ of $\M$ is defined as follows:
\begin{align*}
   & \M, w \modelsmdl p               & \text{ iff } &  & p \in V(w)\text{,}                                               \\
   & \M, w \modelsmdl O \phi          & \text{ iff } &  & \{w' \mid \M, w' \modelsmdl \phi\} \in N(w)\text{,}              \\
   & \M, w \modelsmdl \neg \phi       & \text{ iff } &  & \M, w \not\modelsmdl \phi\text{,}                                \\
   & \M, w \modelsmdl \phi \land \psi & \text{ iff } &  & \M, w \modelsmdl \phi \text{ and } \M, w \modelsmdl \psi\text{.}
\end{align*}

\begin{remark}
  \label{remark:monadicVSdyadic}
  MDL obligations are monadic. A naive conditional obligation can be defined via material implication
  as $O(\phi | \psi) := \psi \rightarrow O\phi$. This yields detachment, since $O(\phi | \psi) \land
    \psi \rightarrow O\phi$, but it behaves inadequately when the obligation essentially depends on the
  condition. In particular, because $\to$ is true whenever its antecedent is false, $O(\phi | \psi)$
  becomes true in all states where $\psi$ does not hold.
\end{remark}

\section{A Tiny Fragment of MDL}
We identify a fragment of MDL that admits reasoning within classical propositional logic. In this
fragment, obligations and prohibitions apply only to atomic normative contents rather than
arbitrary formulas. This restriction is inspired by the deontic theories of the Sanskrit
philosophical school of M{\={\i}}m{\={a}}{\d{m}}s{\={a}}, where commands are understood as having a
single goal~\cite{DBLP:journals/jolli/BerkelCFGO23}, and it is also in line with the deontic
extension DELX of Answer Set Programming~\cite{DBLP:conf/jelia/CabalarCT23}.

More precisely, we define the \emph{tiny deontic language} $\langtiny$ as $$ \phi := \omega \mid
  \neg \phi \mid \phi \land \phi \qquad \text{ and } \qquad \omega := p \mid O \omega \mid F \omega
$$ where $p \in \props$. We call the subset defined by $\omega ::= p \mid O \omega \mid F \omega$
\emph{simple normative atoms} and denote by $\atomstiny$ the set of all such atoms.

\begin{remark}
  Compared to the full language $\langmdl$, $\langtiny$ excludes obligations over Boolean
  combinations (and therefore also over the constants $\top$ or $\bot$). In particular, formulas such
  as $O(\phi \land \psi)$ are not well formed; instead, we use $O\phi \land O\psi$. Note that in
  Chellas' semantics, $O(\phi \land \psi)$ implies $O\phi \land O\psi$, but not conversely. Also
  obligations over disjunctions, such as $O(\phi \lor \psi)$, are likewise not admitted. An
  additional advantage of this restricted formulation is that it blocks problematic inferences such
  as Ross' paradox~\cite{Ross1941}, according to which the obligation to mail a letter entails the
  obligation to mail or burn it.
\end{remark}

\subsubsection{Classical Logic as Alternative Semantics}
We show that reasoning in this fragment reduces to reasoning in classical propositional logic.
Intuitively, the only propositionally valid implications between simple normative atoms are
identities of the form $\omega \to \omega$ (as $\bot \not\in \langtiny$, the MDL axiom $\neg O
  \bot$ is not in the fragment).

For the formal proof we provide a canonical model ($\mtiny$) for the fragment, with one world for
each possible propositional valuation over the set $\atomstiny$ of simple normative atoms. The
lemmas are proven in the appendix.

\begin{defnrep}
  $\mtiny$ is $(W^{\textit{t}},N^{\textit{t}},V^{\textit{t}})$ with $W^t = 2^{\atomstiny}$, and for
  all $w \in W^t$, $V^t(w) = w \cap \props$, and $N^t(w) = closure(w^O \cup w^F)$, where
  \begin{itemize}
    \item $w^O = \{\{w' \in W^t \mid \omega \in w'\} \mid O \omega \in w\}$,
    \item $w^F = \{\{w' \in W^t \mid \omega \not\in w'\} \mid F \omega \in w\}$, and
    \item $closure(X) = \{U \subseteq W^t \mid \text{there exists an } x \in X \text{ such that } x \subseteq
            U\}$.
  \end{itemize}
\end{defnrep}

\begin{lemrep}\label{lem:mtiny1}
  $\mtiny$ is a Chellas-style neighborhood model.
\end{lemrep}

\begin{appendixproof}
  We need to show that $\mtiny$ satisfies the frame properties. Clearly, for each $\omega \in
    \atomstiny$, the set $\{w' \in W^t \mid \omega \in w'\}$ is non-empty (since $W^t$ contains all
  subsets of $\atomstiny$). Thus, $\emptyset \notin w^O$. Analogously, for each $\omega \in
    \atomstiny$, the set $\{w' \in W^t \mid \omega \notin w'\}$ is non-empty. Thus, $\emptyset \notin
    w^F$. Intuitively, $closure(w^O \cup w^F)$ closes $w^O \cup w^F$ minimally under supersets. The
  superset frame property follows immediately. Since $\emptyset \notin w^O \cup w^F$, it follows that
  $\emptyset \notin closure(w^O \cup w^F)$.
\end{appendixproof}

\begin{lemrep}\label{lem:mtiny2}
  Let $w \in W^t$ and $\omega \in \atomstiny$. Then $\mtiny, w \modelsmdl \omega$ iff $\omega \in w$.
\end{lemrep}

\begin{appendixproof}
  By structural induction. Base case: $\omega = p$ for some $p \in \props$. By definition, $p \in
    V^t(w)$ iff $p \in w$. Thus, $\mtiny, w \modelsmdl p$ iff $p \in w$.

  For the induction step, we assume that $\mtiny, w \modelsmdl \omega$ iff $\omega \in w$ for some
  fixed $\omega$. We therefore know that $\{w' \in W^t \mid \omega \in w'\}$ contains exactly all $w
    \in W^t$ such that $\mtiny,w \modelsmdl \omega$, and that $\{w' \in W^t \mid \omega \notin w'\}$
  contains exactly all $w \in W^t$ such that $\mtiny,w \modelsmdl \neg \omega$.

  Thus, $O\omega \in w$ implies that $\{w' \in W^t \mid \mtiny, w' \modelsmdl \omega\} \in N^t(w)$,
  meaning $\mtiny, w \modelsmdl O \omega$. Analogously, $F\omega \in w$ implies that $\{w' \in W^t
    \mid \mtiny, w' \modelsmdl \neg\omega\} \in N^t(w)$, meaning $\mtiny, w \modelsmdl F \omega$.

  For the converse, assume first that $\mtiny, w \modelsmdl O\omega$. By the induction hypothesis,
  $\{w' \in W^t \mid \mtiny,w' \modelsmdl \omega\} = \{w' \in W^t \mid \omega \in w'\}$. Hence $\{w'
    \in W^t \mid \omega \in w'\} \in N^t(w)=closure(w^O \cup w^F)$. So there is some $x \in w^O \cup
    w^F$ such that $x \subseteq \{w' \in W^t \mid \omega \in w'\}$. Now $x$ cannot lie in $w^F$, since
  for every $\psi \in \atomstiny$ the world $\emptyset$ is in $\{w' \in W^t \mid \psi \notin w'\}$,
  but $\emptyset \notin \{w' \in W^t \mid \omega \in w'\}$. Thus $x \in w^O$, so $x=\{w' \in W^t \mid
    \psi \in w'\}$ for some $\psi$ with $O\psi \in w$. If $\psi \neq \omega$, then the world $\{\psi\}$
  belongs to $x$ but not to $\{w' \in W^t \mid \omega \in w'\}$, contradicting $x \subseteq \{w' \in
    W^t \mid \omega \in w'\}$. Therefore $\psi=\omega$, and hence $O\omega \in w$.

  Analogously, assume that $\mtiny, w \modelsmdl F\omega$. By the induction hypothesis, $\{w' \in W^t
    \mid \mtiny,w' \modelsmdl \neg \omega\} = \{w' \in W^t \mid \omega \notin w'\}$. Hence $\{w' \in
    W^t \mid \omega \notin w'\} \in N^t(w)=closure(w^O \cup w^F)$. So there is some $x \in w^O \cup
    w^F$ such that $x \subseteq \{w' \in W^t \mid \omega \notin w'\}$. Now $x$ cannot lie in $w^O$,
  since for every $\psi \in \atomstiny$ the world $\{\psi,\omega\}$ is in $\{w' \in W^t \mid \psi \in
    w'\}$, but not in $\{w' \in W^t \mid \omega \notin w'\}$. Thus $x \in w^F$, so $x=\{w' \in W^t \mid
    \psi \notin w'\}$ for some $\psi$ with $F\psi \in w$. If $\psi \neq \omega$, then the world
  $\{\omega\}$ belongs to $x$ but not to $\{w' \in W^t \mid \omega \notin w'\}$, again a
  contradiction. Therefore $\psi=\omega$, and thus $F\omega \in w$.
\end{appendixproof}

By Lemma~\ref{lem:mtiny2}, every simple normative atom $\omega \in \atomstiny$ is true in $\mtiny$
at $w$ exactly if $\omega \in w$. Since negation and conjunction have the usual classical truth
conditions, it follows by structural induction that truth in $\mtiny$ agrees with truth in
classical logic for all formulas of $\langtiny$. Henceforth we use $\modelspl$ for the truth
relation in classical propositional logic over the set of propositions $\atomstiny$:

\begin{thm}\label{thm:prop}
  Let $w \subseteq \atomstiny$ and $\phi \in \langtiny$. Then $\mtiny, w \modelsmdl \phi$ iff $w
    \modelspl \phi$.
\end{thm}
\noindent
Although $\atomstiny$ is infinite, every concrete finite set of formulas contains only finitely
many simple normative atoms. Hence each instance of satisfiability or entailment in $\langtiny$
reduces to classical propositional reasoning over a finite set of atoms.

\section{Lightweight Deontic Logic}
We now extend the previous fragment to a deontic logic that better captures the conditional nature
of normative statements: Lightweight Deontic Logic (LDL). Encoding conditional norms via material
implication is inadequate (see Remark~\ref{remark:monadicVSdyadic}). Instead, LDL assumes
\emph{explicit conditional norms} over atomic propositions, specified directly by the system
designer. By factual detachment, these conditional norms give rise to unconditional obligations and
prohibitions corresponding to the tiny fragment of MDL, yielding the \emph{implicit unconditional
  norms} of the system. To formalize the distinction of explicit vs implicit norms, we introduce the
language $\langl$ below:
\begin{align*}
  \phi   & ::= \omega \mid \neg \phi \mid \phi \land \phi\text{,}                                       \\
  \omega & ::= p \in \props \mid O\omega \mid F\omega \mid \O(\omega|\phi) \mid \F(\omega|\phi)\text{.}
\end{align*}
We read $\O(\omega|\phi)$ as ``$\omega$ is explicitly obligatory given $\phi$'' and
$\F(\omega|\phi)$ as ``$\omega$ is explicitly forbidden given $\phi$''. We call a formula of the
form $\omega ::= p \mid O\omega \mid F\omega \mid \O(\omega|\phi) \mid \F(\omega|\phi)$ a
\emph{normative atom} and denote by $\atoms$ the set of all such atoms. Note that $\atoms$ contains
higher-order norms. Consider $\O(F\,smoking|F\,drinking)$, for example.

\subsection{Propositional Semantics}
Although $\langl$ generalizes $\langtiny$ ($\phi \in \langl$ for each $\phi \in \langtiny$), as
shown below, it can still be characterized by the semantics of classical logic.

The introduced semantics makes a conceptual distinction between explicit and implicit normative
atoms. We say that a given explicit conditional obligation $\O(\omega|\phi)$ or prohibition
$\F(\omega|\phi)$ is \emph{in effect} if its condition $\phi$ is satisfied. This induces the
corresponding implicit obligation $O \omega$ or prohibition $F \omega$ to be true (which concerns
the actual state rather than hypothetical situations).

\begin{defn}\label{def:ldlconsistency}
  An \emph{LDL model} is a propositional valuation $M \subseteq \Omega$ satisfying for all $\omega
    \in \Omega$ and $\phi \in \langl$ the \emph{LDL-consistency} conditions:\,\footnote{We can
    equivalently write $M$ as a Boolean function $\{\omega \mapsto \top \mid \omega \in M\} \cup
      \{\omega \mapsto \bot \mid \omega \not\in M\}$. For $\omega \in \Omega$, we write $M \modelspl
      \omega$ iff $\omega \in M$, or equivalently iff $M(\omega) = \top$.}
  \begin{itemize}
    \item $\O(\omega|\phi) \in M \Longrightarrow M \not\modelspl \phi \text { or } O\omega \in M$, and
    \item $\F(\omega|\phi) \in M \Longrightarrow M \not\modelspl \phi \text { or } F\omega \in M$.
  \end{itemize}
\end{defn}

\begin{defn}\label{def:ldlsemantics}
  Let $\phi \in \langl$ and let $M$ be an LDL model. We write $M \modelsl \phi$ if and only if $M
    \modelspl \phi$.
\end{defn}

Intuitively, the additional consistency requirement allows us to infer (unconditional) implicit
obligations from (conditional) explicit obligations via factual detachment, without strengthening
of the antecedent for explicit obligations, as shown by the example below.
\begin{remark}
  Like in MDL, having a dilemma (i.e. both $O \omega$ and $F \omega$) does not lead to a
  contradiction.
\end{remark}

\begin{ex}
  Consider $\phi = p \land \O(r|p)$ and $M = \{p, \O(r|p), Or\}$. Clearly, $M$ is an LDL model and $M
    \modelspl \phi \land \neg \O(r|p \land q)$. Therefore, $M \modelsl \phi \land \neg \O(r|p \land
    q)$. \end{ex}

\begin{remark}
  $\ldl$ is a conservative extension of the tiny fragment of Chellas' MDL. \end{remark}

\subsection{Axiomatization of $\ldl$}
We define a Hilbert-style proof system for $\ldl$. Its axiom schemata are:
\begin{description}[labelwidth=1.7cm]
  \item[(Prop)] Every instance of a classical propositional tautology over $\atoms$.
  \item[(O-Det)] $\O(\omega|\phi) \land \phi \to O\omega$, for all $\omega \in \atoms$ and $\phi \in \langl$.
  \item[(F-Det)] $\F(\omega|\phi) \land \phi \to F\omega$, for all $\omega \in \atoms$ and $\phi \in \langl$.
\end{description}
The only inference rule is modus ponens: {\bf (MP)} From $\phi$ and $\phi \to \psi$, infer $\psi$.

For $\Gamma \subseteq \langl$ and $\phi \in \langl$, we write $\Gamma \vdashl \phi$ iff there
exists a finite sequence of formulas $\phi_1,\dots,\phi_n \in \langl$ with $\phi_n=\phi$ such that
each $\phi_i$ is either a member of $\Gamma$, an instance of one of the above axiom schemata, or
obtained from earlier formulas in the sequence by \textup{(MP)}.

To prove soundness and completeness, we first provide a lemma connecting semantic consequence in
LDL-theories to that of classical propositional logic. For $\Gamma \subseteq \langl$ and $\phi \in
  \langl$, we write $\Gamma \impliesl \phi$ if every LDL model $M$ satisfying $M \modelsl \Gamma$
also satisfies $M \modelsl \phi$. Similarly, we write $\Gamma \impliespl \phi$ if $M \modelspl
  \phi$ for each $M \modelspl \Gamma$. We abbreviate $\emptyset \impliesl \phi$ by $\impliesl \phi$
and $\emptyset \impliespl \phi$ by $\impliespl \phi$.

\begin{lem}\label{lem:theories}
  Let $\Delta$ be the set of all instances of $\O(\omega|\phi) \land \phi \to O \omega$ and
  $\F(\omega|\phi) \land \phi \to F \omega$. Then for all $\Gamma \subseteq \langl$ and $\phi \in
    \langl$, we have
  \[ \Gamma \impliesl \phi \iff \Gamma \cup \Delta \impliespl \phi\text{.} \]
\end{lem}

\begin{proof}
  We first note that a propositional valuation $M \subseteq \atoms$ is an LDL model iff $M \modelspl
    \Delta$. Indeed, considering only obligations (the prohibition case is analogous), $M$ violates
  LDL-consistency exactly when there exists $\O(\omega | \phi)$ such that $M \modelspl \phi$,
  $O\omega \notin M$, and thus $M \not\modelspl (\O(\omega | \phi) \land \phi \rightarrow O\omega)$.
  Hence, LDL-inconsistency corresponds precisely to falsifying some formula in $\Delta$, and vice
  versa. It follows that the models of $\Gamma \cup \Delta$ are exactly the LDL models of $\Gamma$.
  Therefore, by Definition~\ref{def:ldlsemantics}, $\Gamma \impliesl \phi \iff \Gamma \cup \Delta
    \impliespl \phi$.
\end{proof}

\begin{thm}[Soundness and Completeness] Let $\Gamma, \phi \subseteq \langl$. Then
  \[ \Gamma \vdashl \phi \iff \Gamma \impliesl \phi\text{.}\]
\end{thm}
\begin{proof}
  Let $\Delta$ be defined as in Lemma~\ref{lem:theories}. We know that $\Gamma \impliesl \phi$ iff
  $\Gamma \cup \Delta \impliespl \phi$. By soundness and completeness of classical propositional
  logic, $\Gamma \cup \Delta \impliespl \phi$ iff $\Gamma \cup \Delta \vdashpl \phi$ (where
  $\vdashpl$ is derivability in classical propositional logic over $\Omega$). Finally, consider that
  the Hilbert system is exactly classical propositional logic together with all formulas in $\Delta$
  as additional axiom instances. Therefore, $\Gamma \cup \Delta \vdashpl \phi \iff \Gamma \vdashl
    \phi$. Thus, $\Gamma \impliesl \phi \iff \Gamma \vdashl \phi$.
\end{proof}

Lemma~\ref{lem:theories} yields a polynomial reduction from LDL-satisfiability and LDL-validity to
classical propositional satisfiability and validity over the normative atoms occurring in the input
formula. For $\phi\in\langl$, let $\Delta^\phi$ contain the detachment implications corresponding
to all explicit normative atoms occurring recursively in $\phi$. Then $\phi$ is LDL-satisfiable iff
$\phi\land\bigwedge\Delta^\phi$ is classically satisfiable, and LDL-valid iff
$(\bigwedge\Delta^\phi)\to\phi$ is classically valid. Hence:

\begin{thm}
  Satisfiability in LDL is NP-complete, and validity is co-NP-complete.
\end{thm}

\begin{toappendix}
  \begin{thm}
    \label{Th:conservativity}
    Let $\Gamma \subseteq \langtiny$ and $\phi \in \langtiny$. Then
    \[ \Gamma \impliesmdl \phi \iff \Gamma \impliesl \phi\text{.}\]
  \end{thm}

  \begin{proof}
    We prove both directions by contraposition.

    Let $\Gamma \not\impliesmdl \phi$. Then there is a pointed MDL model $(\M,w)$ such that
    $\M,w\modelsmdl\Gamma$ and $\M,w\not\modelsmdl\phi$. Put
    \[
      u=\{\omega\in\atomstiny\mid\M,w\modelsmdl\omega\}.
    \]
    Since formulas in $\langtiny$ are Boolean combinations of atoms in $\atomstiny$, structural
    induction gives $u\modelspl\Gamma$ and $u\not\modelspl\phi$. Moreover, $u\subseteq\atoms$ is an LDL
    model because it contains no explicit normative atoms. Thus $u\modelsl\Gamma$ and
    $u\not\modelsl\phi$, so $\Gamma \not\impliesl \phi$.

    Conversely, let $\Gamma \not\impliesl \phi$. Then there is an LDL model $M$ with $M \modelsl
      \Gamma$ and $M \not\modelsl \phi$. Hence $M \not\modelspl \phi$. Let $M' = M \cap \atomstiny$.
    Since $\Gamma \subseteq \langtiny$ and $\phi \in \langtiny$, their truth depends only on the atoms
    in $\atomstiny$. Hence $M' \modelspl \Gamma$ and $M' \not\modelspl \phi$. By
    Theorem~\ref{thm:prop}, $\mtiny,M' \modelsmdl \Gamma$ and $\mtiny,M' \not\modelsmdl \phi$.
    Therefore, $\Gamma \not\impliesmdl \phi$.
  \end{proof}
\end{toappendix}

\section{A Doxastic Extension of LDL}
\label{sec:doxastic}
We now define a direct combination of LDL with multi-agent doxastic logic
KD45~\cite{Hintikka1962,DBLP:journals/ai/HalpernM92}. This allows us to express norms conditioned
on what agents believe and what they believe to be possible, as well as to express beliefs about
norms.

Let $\agents$ be a finite set of agents. The language $\langb$ is defined by the following grammar,
where $p \in \props$ and $i \in \agents$:
\[
  \begin{aligned}
    \phi   & ::= \omega \mid B_i\phi \mid \neg \phi \mid \phi \land \phi,               \\
    \omega & ::= p \mid O\omega \mid F\omega \mid \O(\omega|\phi) \mid \F(\omega|\phi).
  \end{aligned}
\]
We read formulas $B_i \phi$ as \emph{``agent $i$ believes $\phi$''} and define $\hat{B}_i \phi$ as
an abbreviation for $\neg B_i \neg \phi$ (which can be read ``agent $i$ believes it possible that
$\phi$''). We further call formulas generated by $\omega$ \emph{objective atoms} and write
$\atomsb$ for the set of all such atoms. An important difference from LDL is that conditions of
explicit obligations and prohibitions may now contain belief operators. The normative contents of
obligations and prohibitions, however, are still restricted to formulas generated by $\omega$ (in
particular, we cannot model the obligation or prohibition to have a certain belief). Moreover, in
$\langb$, we can also specify beliefs about normative facts, such as $B_iO\omega$, $B_iF\omega$, or
$B_i\O(\omega|\phi)$.

\subsection{Kripke Semantics over LDL-Valuations}
The $\ldlb$ semantics combines doxastic Kripke models with LDL-consistency.

\begin{defn}
  An \emph{$\ldlb$ premodel} is a tuple
  \[
    \M=(W,(R_i)_{i\in\agents},V)
  \]
  where $W$ is a non-empty set of worlds, $R_i\subseteq W\times W$ is agent $i$'s doxastic
  \emph{accessibility relation}, and $V:W\to 2^{\atomsb}$ assigns an LDL-style valuation to each
  world. For $x\in W$, write $R_i(x)=\{y\in W\mid xR_i y\}$. Truth in a premodel is defined as
  follows:
  \begin{align*}
     & \M,w\modelsk \omega
     &                            & \text{ iff }\omega\in V(w),                                       \\
     & \M,w\modelsk B_i\phi
     &                            & \text{ iff }\M,v\modelsk\phi\text{ for all }v\text{ with }wR_i v, \\
     & \M,w\modelsk \neg\phi
     &                            & \text{ iff }\M,w\not\modelsk\phi,                                 \\
     & \M,w\modelsk \phi\land\psi
     &                            & \text{ iff }\M,w\modelsk\phi\text{ and }\M,w\modelsk\psi.
  \end{align*}

  A premodel is an $\ldlb$ model iff it satisfies the following two conditions.
  \begin{enumerate}[label=(\roman*)]
    \item For every $i \in \agents$, $R_i$ is \emph{serial}, \emph{transitive}, and \emph{euclidean}.
    \item For every $w\in W$, every $\omega\in\atomsb$, and every $\phi\in\langb$,
          \begin{align*}
            \O(\omega|\phi)\in V(w) \text{ and } \M,w\modelsk\phi
             & \quad\Longrightarrow\quad O\omega\in V(w), \\
            \F(\omega|\phi)\in V(w) \text{ and } \M,w\modelsk\phi
             & \quad\Longrightarrow\quad F\omega\in V(w).
          \end{align*}
  \end{enumerate}

\end{defn}

The first condition gives the doxastic logic KD45 for each agent: beliefs are consistent; that is,
$B_i\phi \to \hat{B}_i \phi$ is valid. Furthermore, agents have positive and negative
introspection: $B_i \phi \to B_i B_i \phi$ and $\neg B_i \phi \to B_i \neg B_i
  \phi$.\footnote{Positive and negative introspection are controversial assumptions. In the case of
  artificial agents, however, they can be justified under the assumption that each agent has a
  consistent belief base, full access, and the ability to reason perfectly about it.}

The second condition enforces \emph{LDL-consistency} at each individual world. This effectively
lifts the factual detachment of LDL to the mental states of the agents: If an agent believes (or
believes it possible) that at the same time an explicit norm is given and its condition is
satisfied, then the agent believes (or believes it possible) that the norm is in effect.

For $\Gamma\subseteq\langb$ and $\phi \in \langb$, we write $\Gamma\impliesk\phi$ iff every world
of every $\ldlb$ model satisfying all formulas in $\Gamma$ also satisfies $\phi$.

\begin{ex}\label{ex:vehicle-ldlb}
  Let $r$ mean that an autonomous car moves to the right to yield (e.g., if an emergency vehicle is
  approaching), and $c$ that a cyclist is to the right of the car. The language distinguishes $\F(r|
    B_a c)$ from $\F(r| \hat{B}_a c)$; only the latter captures ``if a cyclist cannot be ruled out,
  moving right is forbidden'', that is, $$\{\F(r| \hat{B}_a c),\hat{B}_a c\}\impliesk F r \quad \mbox{while} \quad
    \{\F(r| B_a c),\hat{B}_a c\}\not\impliesk F r.$$

  We can also reason about an agent's own perspective (i.e., $a$ believes both the norm $\F(r|
    \hat{B}_a c)$ and that $c$ is possible): $\{B_a\F(r| \hat{B}_a c),\hat{B}_a c\} \impliesk B_a F r$
  (via $B_a \hat{B}_a c$ by introspection), but $\{\F(r| \hat{B}_a c),\hat{B}_a c\}\not\impliesk B_a
    F r.$

\end{ex}

\begin{remark}
  $\ldlb$ can be axiomatized by adding the LDL detachment axioms to multi-agent KD45. The additional
  axiom schemata are:
  \begin{description}[labelwidth=1.7cm]
    \item[(O-Det)] $\O(\omega|\phi)\land\phi\to O\omega$, for all $\omega\in\atomsb$ and $\phi\in\langb$.
    \item[(F-Det)] $\F(\omega|\phi)\land\phi\to F\omega$, for all $\omega\in\atomsb$ and $\phi\in\langb$.
  \end{description}
\end{remark}

\begin{toappendix}
  The full axiomatization for $\ldlb$ is given as follows:

  \begin{description}[labelwidth=1.7cm]
    \item[(Prop)] Every instance of a classical propositional tautology.
    \item[(K)] $B_i(\phi\to\psi)\to(B_i\phi\to B_i\psi)$, for all $i\in\agents$.
    \item[(D)] $B_i\phi\to \hat B_i\phi$, for all $i\in\agents$.
    \item[(4)] $B_i\phi\to B_iB_i\phi$, for all $i\in\agents$.
    \item[(5)] $\neg B_i\phi\to B_i\neg B_i\phi$, for all $i\in\agents$.
    \item[(O-Det)] $\O(\omega|\phi)\land\phi\to O\omega$, for all $\omega\in\atomsb$ and $\phi\in\langb$.
    \item[(F-Det)] $\F(\omega|\phi)\land\phi\to F\omega$, for all $\omega\in\atomsb$ and $\phi\in\langb$.
  \end{description}
  The rules are modus ponens and necessitation:
  \begin{description}[labelwidth=1.7cm]
    \item[(MP)] From $\phi$ and $\phi\to\psi$, infer $\psi$.
    \item[(Nec)] From a theorem $\phi$, infer $B_i\phi$, for all $i\in\agents$.
  \end{description}

  We write $\vdashk\phi$ for theoremhood in this system. For $\Gamma\subseteq\langb$, we write
  $\Gamma\vdashk\phi$ iff there are $\gamma_1,\dots,\gamma_n\in\Gamma$ such that
  \[
    \vdashk (\gamma_1\land\cdots\land\gamma_n)\to\phi,
  \]
  where $n=0$ means $\vdashk\phi$.

  \begin{thm}\label{thm:ldlb-completeness}
    Let $\Gamma\subseteq\langb$ and $\phi\in\langb$. Then
    \[
      \Gamma\vdashk\phi \quad\text{iff}\quad \Gamma\impliesk\phi\text{.}
    \]
  \end{thm}

  \begin{proof}
    Soundness is immediate. The propositional and KD45 axioms are valid on serial, transitive, and
    euclidean frames, and the two detachment axioms are valid by the second condition in the definition
    of $\ldlb$ models. The rules \textup{(MP)} and \textup{(Nec)} preserve validity, and the
    finite-premise definition of $\Gamma\vdashk\phi$ gives soundness for local consequence from
    theories.

    For completeness, use the standard canonical model construction for multi-agent KD45, but with
    atoms $\atomsb$. Let $X$ be the set of all maximal $\ldlb$-consistent sets. Define
    \[
      xR_i y \quad\text{iff}\quad \{\psi \mid B_i\psi\in x\}\subseteq y,
      \qquad
      V(x)=x\cap\atomsb.
    \]

    We first check the frame conditions. Seriality follows from axiom \textup{(D)}. Indeed, for any
    $x\in X$, the set
    \[
      \Delta_i(x)=\{\psi\mid B_i\psi\in x\}
    \]
    is consistent: otherwise, for some $\psi_1,\dots,\psi_n\in\Delta_i(x)$, we would have $\vdashk
      (\psi_1\land\cdots\land\psi_n)\to\bot$. By necessitation and \textup{(K)}, $\vdashk
      B_i\psi_1\land\cdots\land B_i\psi_n\to B_i\bot$, so $B_i\bot\in x$. But \textup{(D)} gives
    $B_i\bot\to\hat B_i\bot$, while necessitation gives $B_i\top$; contradiction. Hence $\Delta_i(x)$
    extends to some $y\in X$, and $xR_i y$.

    Transitivity follows from \textup{(4)}. If $xR_i y$ and $yR_i z$, and $B_i\psi\in x$, then
    $B_iB_i\psi\in x$, so $B_i\psi\in y$, and therefore $\psi\in z$. Thus $xR_i z$.

    Euclideanness follows from \textup{(5)}. If $xR_i y$ and $xR_i z$, and $B_i\psi\in y$, then
    $B_i\psi\in x$; otherwise $\neg B_i\psi\in x$, so by \textup{(5)}, $B_i\neg B_i\psi\in x$, whence
    $\neg B_i\psi\in y$, a contradiction. Since $B_i\psi\in x$ and $xR_i z$, we get $\psi\in z$. Thus
    $yR_i z$.

    The standard canonical-model argument for normal modal logics gives the truth lemma:
    \[
      \M^c,x\modelsk\chi \quad\text{iff}\quad \chi\in x
    \]
    for every $\chi\in\langb$.

    It remains only to check the LDLB detachment condition. Suppose $\O(\omega|\psi)\in V(x)$ and
    $\M^c,x\modelsk\psi$. By the truth lemma, $\psi\in x$. Since
    \[
      \O(\omega|\psi)\land\psi\to O\omega
    \]
    is an axiom instance, it belongs to $x$; hence $O\omega\in x$, and therefore $O\omega\in V(x)$. The
    proof for prohibitions is identical. Thus the canonical model is an $\ldlb$ model.

    Now assume $\Gamma\not\vdashk\phi$. Then $\Gamma\cup\{\neg\phi\}$ is $\ldlb$-consistent; otherwise
    some finite subset of $\Gamma$ would prove $\phi$. Extend $\Gamma\cup\{\neg\phi\}$ to a maximal
    consistent set $x\in X$. By the truth lemma, $\M^c,x\modelsk\Gamma$ and $\M^c,x\not\modelsk\phi$.
    Hence $\Gamma\not\impliesk\phi$. This proves completeness.
  \end{proof}
\end{toappendix}

\section{Optimizing Norm Violations}
We now turn to decision making itself. An acting agent $a \in\agents$ must choose among several
possible decisions, each of which may violate some of the applicable norms, each having a possibly
different weight. We formalize this as the Doxastic Norm Compliance Optimization Problem (DNCOP):
given a finite description of the facts, the agents' beliefs, and a set of weighted norms, find a
decision that minimizes total violation weight.

Two perspectives are possible. The objective perspective evaluates violations against the full,
true state of the world; the subjective perspective evaluates them against what the acting agent
believes. We provide conditions under which these two agree and under which optimal decisions can
be found by a polynomial reduction to weighted partial MaxSAT~\cite{DBLP:conf/sat/AnsoteguiBL09}.

\subsection{State Vocabulary}\label{sec:vocabulary}
A DNCOP is described using a fixed, finite set of atoms -- the \emph{state vocabulary}. Besides
plain facts, we need a compact way to talk about nested beliefs. Following restricted modal
literals and proper epistemic knowledge
bases~\cite{DBLP:conf/ecai/LakemeyerL12,DBLP:conf/aaai/MuiseBFMMPS15}, beliefs are represented only
by nested belief operators applied to a single objective atom.

As agents' beliefs are introspective, a nested belief such as $B_1 B_1 B_2 \neg B_2 B_2 B_1 \omega$
is LDLB-equivalent to $B_1 \neg B_2 B_1 \omega$: repeated operators for the same agent collapse. We
therefore only need belief chains that alternate agents.

\begin{defn}\label{def:rdda}\label{def:finite-vocabulary}
  A \emph{restricted doxastic-deontic atom (RDDA)} is either an objective atom $\omega\in\atomsb$, or
  a formula $X_1\ldots X_n\omega$ ($n \geq 1$), where each $X_j$ is either $B_{a_j}$ or
  $B_{a_j}\neg$, consecutive operators refer to different agents ($a_j\neq a_{j+1}$), and $\omega$ is
  an objective atom. We denote the set of all RDDAs by $\rdda$.

  A finite set $\statevocab$ of such atoms is a \emph{state vocabulary} if it is closed under
  sub-atoms, i.e., (let $\mathsf{Bool}(\statevocab)$ denote all Boolean combinations over members of
  $\statevocab$)
  \begin{enumerate}[label=(\roman*)]
    \item if $B_i \alpha \in \statevocab$ or $B_i \neg \alpha \in \statevocab$ (for any RDDA $\alpha$), then
          $\alpha \in \statevocab$,
    \item if $O\omega\in\statevocab$, $F\omega\in\statevocab$, $\O(\omega|\phi)\in\statevocab$, or
          $\F(\omega|\phi)\in\statevocab$, then $\omega\in\statevocab$; and
    \item if $\O(\omega|\phi)\in\statevocab$ or $\F(\omega|\phi)\in\statevocab$, then
          $\phi\in\mathsf{Bool}(\statevocab)$.
  \end{enumerate}

  Finally, we denote by $\siga$ the subset of $\statevocab$ consisting of RDDAs that have $B_a$ as
  their outermost operator (i.e., the beliefs of the acting agent).
\end{defn}

A state vocabulary therefore contains the atoms needed to finitely describe a norm compliance
optimization problem. Later, these atoms will be classified as fixed by the environment or
controlled by the acting agent.

\subsection{Doxastic Norm Compliance Optimization Problems (DNCOPs)}\label{sec:dncop}
We formulate the optimization problem faced by a fixed acting agent $\actagent\in\agents$. The goal
is to choose a decision that minimizes (weighted) norm violations.

The state vocabulary $\statevocab$ is partitioned into atoms fixed true by the environment
$\trueatoms$ (e.g., whether an emergency vehicle is approaching), fixed false $\falseatoms$, and
$\dbar$ atoms under the agent's control (its decisions and their knowledge\footnote{Since LDLB does
  not have knowledge, we use true beliefs.}).
\begin{defn}\label{def:dncop}
  A \emph{DNCOP} is a tuple $\Delta=\dncoptuple$, where
  \begin{itemize}
    \item $\statevocab$ is a finite state vocabulary,
    \item $\trueatoms\subseteq\statevocab$ (resp. $\falseatoms\subseteq\statevocab$) contains the atoms fixed
          true (resp. false) independently of the decision the agent makes,
    \item $\decatoms\subseteq\atomsb$ is the set of \emph{decision atoms} the acting agent sets directly, and
          we write $\dbar = \decatoms\cup \{B_{\actagent}\delta,
            B_{\actagent}\neg\delta\mid\delta\in\decatoms\}$ for these atoms together with the agent's
          (automatically correct) beliefs about them;
    \item $\decconstraint\in\mathsf{Bool}(\decatoms)$ is the \emph{decision constraint}, defining the
          possible decisions\\ $ \decisions{\Delta} = \{d\subseteq\decatoms\mid d\modelspl\decconstraint\} $
    \item $\countednorms\subseteq\statevocab$ is the set of \emph{relevant norms} (of the form
          $\O(\omega|\phi)$, $\F(\omega|\phi)$, $O\omega$, $F\omega$).
    \item $v$ is a function that assigns a positive integer weight to each norm $n \in \countednorms$.
  \end{itemize}
  We require that $\statevocab=\trueatoms\mathbin{\dot\cup}\falseatoms\mathbin{\dot\cup}\dbar $ is a
  disjoint union.
\end{defn}

Unless specified otherwise, we assume $\decconstraint=\top$, so every subset of $\decatoms$ is a
possible decision. We also omit $\statevocab$ when clear from the context.

\begin{ex}\label{ex:vehicle-factual-specification}
  Let $r$ and $c$ be as in Example~\ref{ex:vehicle-ldlb}, and let $e$ denote that an emergency
  vehicle is approaching from behind. Consider the norms $\O(r|e)$ (``move right if an emergency
  vehicle approaches'') and $\F(r|c)$ (``don't move right'', once a cyclist is present). Fix $e$ and
  the agent's correct belief in it; leave open whether $c$ is fixed true or false. Assume $v(\O(r|e))
    = 1$ and $v(\F(r | c)) = 2$. We define two variants of the scenario, one where no cyclist is
  present ($c$ is false and fixed) and one where a cyclist is present ($c$ is true and fixed). We
  define the DNCOP where there is a cyclist as $\Delta_c$ (with $\trueatoms_c$, $\falseatoms_c$), and
  the DNCOP without a cyclist as $\Delta_{\neg c}$ (with $\trueatoms_{\neg c}$, $\falseatoms_{\neg
      c}$). For both versions $\decatoms =\{r\}$, and $\countednorms=\{\O(r|e), \F(r | c)\}$.
  \begin{align*}
    \trueatoms_{\neg c} & =
    \{e,\O(r|e),\F(r|c), Or, B_{\actagent}e,B_{\actagent}\O(r|e),B_{\actagent}\F(r|c)\},\
    \falseatoms_{\neg c} =
    \{c,Fr,B_{\actagent}c,B_{\actagent}\neg c\}, \\
    \trueatoms_{c}      & =
    \{e,c,\O(r|e),\F(r|c),Or,Fr, B_{\actagent}e,B_{\actagent}\O(r|e),B_{\actagent}\F(r|c)\},\
    \falseatoms_{c} =
    \{B_{\actagent}c,B_{\actagent}\neg c\}.
  \end{align*}

  In either case, we assume that $a$ believes an emergency vehicle is approaching from behind and
  that both $\O(r|e)$ and $\F(r|c)$ hold, and therefore that $Or$ holds. In $\Delta_c$ also $Fr$
  holds. All these beliefs are \emph{true beliefs}.

  By the default $\decconstraint=\top$, the possible decisions are $ \decisions{\Delta_{\neg c}} =
    \decisions{\Delta_c} = \{\emptyset,\{r\}\} $, where $\emptyset$ represents not moving right and
  $\{r\}$ represents moving right.

\end{ex}

\subsection{Compatible Models and Violation Counts}
\label{sec:compatible}
Given a decision $d$ in a DNCOP $\Delta$, we can ask which models of the LDLB logic are consistent
with it -- either assuming full, correct information (\emph{objective compatibility}) or only what
the acting agent believes (\emph{subjective compatibility}).

\begin{defn}\label{def:decision-description}
  Let $\Delta=\dncoptuple$ be a DNCOP and $d\subseteq\decatoms$. Let
  \begin{itemize}
    \item $\choice{\Delta}{d} = \bigcup_{\delta\in d}\{\delta,B_{\actagent}\delta,\neg
            B_{\actagent}\neg\delta\} \cup \bigcup_{\delta\in\decatoms\setminus d}\{\neg\delta,\neg
            B_{\actagent}\delta,B_{\actagent}\neg\delta\}$, \\ (adding $B_a$ because the agent believes their
          decision and thus also $\hat{B}_a$).
    \item $\objdesc{\Delta}{d} = \trueatoms \cup\{\neg\alpha\mid\alpha\in\falseatoms\}
            \cup\choice{\Delta}{d}$
    \item $\subjdesc{\Delta}{d} = (\trueatoms \cap \siga) \cup \{\neg \alpha \mid \alpha \in \falseatoms \cap
            \siga \} \cup\choice{\Delta}{d}$\\ i.e., $\{$only the acting agent's own beliefs about $\trueatoms$
          and $\falseatoms \} \cup\choice{\Delta}{d}$.
  \end{itemize}
\end{defn}

We characterize the LDLB-models $(\mathcal M, w)$ that are subjectively and objectively compatible
with a DNCOP $\Delta$ and a decision $d$. For the subjective perspective, we assume that the acting
agent trusts their own beliefs. To achieve this, we restrict ourselves to models with
$wR_{\actagent}w$ for the actual world $w$.\footnote{This ensures that $(\mathcal M, w) \modelsk
    \phi$ whenever $(\mathcal M, w) \modelsk B_{\actagent} \phi$. }

\begin{defn}\label{def:realization}
  A pointed $\ldlb$ model $(\M,w)$ is \emph{objectively compatible} with $(\Delta, d)$ if
  $\M,w\modelsk\objdesc{\Delta}{d}$. It is \emph{subjectively compatible} with $(\Delta, d)$ if
  $\M,w\modelsk\subjdesc{\Delta}{d}$ and the agent trusts itself at $w$, i.e., $wR_{\actagent}w$.

  We denote the corresponding sets of models by $\objcomp{\Delta}{d}$ and $\subjcomp{\Delta}{d}$.
  Nonemptiness of the respective set defines objective or subjective admissibility.
\end{defn}

In some DNCOPs none of the decisions are objectively admissible (due to inconsistencies in beliefs,
norms, or both). In other DNCOPs, only specific decisions are. Consider an example from the
perspective of the designer, where $p \in \trueatoms$ is fixed true and $Ox, \O(x|p)\in \decatoms$
are decision atoms. By detachment, making $\O(x|p)$ true requires making $O x$ true as well. In
this example $\{\O(x|p), Ox\}$ is objectively admissible but $\{\O(x|p)\}$ is not. Since facts in
$\trueatoms$ are inaccessible to the acting agent, the objective admissibility of $\{\O(x|p)\}$
cannot be determined subjectively (and in particular is not implied by subjective admissibility).

Hence, we restrict our attention to \emph{decision-sound} DNCOPs. This strong but necessary promise
allows optimization from the agent's perspective.

\begin{defn}\label{def:decision-soundness}
  We call a DNCOP $\Delta$ \emph{decision-sound} if $\decisions{\Delta}\neq\emptyset$ and every
  $d\in\decisions{\Delta}$ is objectively admissible.
\end{defn}

Importantly, as shown below, every objectively admissible decision is also subjectively admissible.
This justifies the use of subjectively compatible models when optimizing from the subjective
perspective.

\begin{prop}\label{prop:objective-subjective-admissibility}
  Let $\Delta=\dncoptuple$ and $d\subseteq\decatoms$. If $d$ is objectively admissible, then it is
  subjectively admissible.
\end{prop}

\begin{proof}
  Choose $(\M,w)\in\objcomp{\Delta}{d}$. By seriality, $wR_{\actagent}v$ for some $v$. Transitivity
  and euclideanness give $R_{\actagent}(w)=R_{\actagent}(v)$ and $vR_{\actagent}v$. Hence every RDDA
  whose outermost agent is $\actagent$ has the same truth value at $w$ and $v$, so the non-decision
  part of $\subjdesc{\Delta}{d}$ holds at $v$. Because $w$ and $v$ have the same
  $\actagent$-successors, all belief literals in $\choice{\Delta}{d}$ that hold at $w$ also hold at
  $v$. Moreover, $B_{\actagent}\delta$ or $B_{\actagent}\neg\delta$ ensures that the corresponding
  decision literal $\delta$ or $\neg\delta$ holds at $v$. Hence $\M,v\modelsk\subjdesc{\Delta}{d}$.
  Together with $vR_{\actagent}v$, this gives $(\M,v)\in\subjcomp{\Delta}{d}$. Therefore
  $\subjcomp{\Delta}{d}\neq\emptyset$, so $d$ is subjectively admissible.
\end{proof}

\label{sec:counting}

We show how to count and optimize norm violations for DNCOPs. We first formalize a violation
condition $\violformula{n}$ for norms: An explicit or implicit obligation is violated when it
exists, its content is false, and (only in the case of explicit obligations) its condition holds.
Prohibitions are treated analogously.

We use this to define the weighted norm violation count $\violcount{\M}{w}$ for a given model
$(\M,w)$, and subsequently the number of violations $\objviol{\Delta,d}$ and $\subjviol{\Delta,d}$
for admissible decisions $d$ in a DNCOP $\Delta$, as predicted by the objective or subjective
perspective. Since there may be multiple compatible models, our definitions use a pessimistic
worst-case\footnote{Assuming the best case is unrealistic and assuming the average is unjustified
  since compatible models need not be equally likely.} assumption

\begin{defn}\label{def:violation-count}
  Let $\Delta=\dncoptuple$ be a DNCOP. For each norm $n\in\countednorms$ we define the violation
  condition as
  \[
    \violformula{n}
    =
    \begin{cases}
      \O(\omega|\phi)\land\phi\land\neg\omega, & n=\O(\omega|\phi), \\
      \F(\omega|\phi)\land\phi\land\omega,     & n=\F(\omega|\phi), \\
      O\omega\land\neg\omega,                  & n=O\omega,         \\
      F\omega\land\omega,                      & n=F\omega.
    \end{cases}
  \]
  For a pointed model $(\M,w)$, we define its weighted norm violation count as
  \[
    \violcount{\M}{w}
    =
    \sum_{\substack{n\in\countednorms\\
        \M,w\modelsk\violformula{n}}}\normweight(n).
  \]
  If $d\subseteq\decatoms$ is objectively admissible, we define its objective violation count as
  \[
    \objviol{\Delta,d}
    =
    \max_{(\M,w)\in\objcomp{\Delta}{d}}
    \violcount{\M}{w}.
  \]
  If $d \subseteq\decatoms$ is subjectively admissible, we define its subjective violation count as
  \[
    \subjviol{\Delta,d}
    =
    \max_{(\M,w)\in\subjcomp{\Delta}{d}}
    \violcount{\M}{w}.
  \]
  Both maxima exist because, by admissibility, the respective sets of compatible models are non-empty
  and every violation count belongs to $\{0,\ldots,\sum_{n\in\countednorms}\normweight(n)\}$.
\end{defn}

\begin{remark} The use of weights allows our framework to handle the priority of more specific norms over more
  general ones (\emph{lex specialis}), and contrary-to-duty (CTD) obligations. Priority between norms
  can be encoded by assigning higher weights to more specific norms, thereby incentivizing the agent
  to comply with them whenever they apply. For CTD obligations, independently of the concrete
  positive weights assigned to the primary obligation (whose violation triggers the secondary
  obligation) and to the secondary obligation, violating both always incurs a higher cost than
  violating only the primary obligation. When multiple interacting norms are present, however,
  appropriate weights must be chosen to reflect the desired priorities, in line with the utilitarian
  approach in~\cite{DBLP:conf/ijcai/NeufeldCT25} to design norm-compliant reinforcement learning
  agents.
\end{remark}

By minimizing over all possible decisions, we now define the objectively and subjectively optimal
number of violations $\objopt{\Delta}$ and $\subjopt{\Delta}$, as well as the objectively and
subjectively optimal decisions $\objopta{\Delta}$ and $\subjopta{\Delta}$. This is well-defined for
decision-sound DNCOPs since they have at least one possible decision and all decisions are
admissible.
\begin{defn}
  Let $\Delta$ be a decision-sound DNCOP. We define its objectively and subjectively optimal norm
  violation counts as
  $$
    \objopt{\Delta}
    =
    \min_{d\in\decisions{\Delta}}\objviol{\Delta,d}, \quad \mbox{and} \quad
    \subjopt{\Delta}
    =
    \min_{d\in\decisions{\Delta}}\subjviol{\Delta,d}.
  $$
  The corresponding sets of optimal decisions are defined as
  \begin{align*}
    \objopta{\Delta}
     & =
    \{d\in\decisions{\Delta}\mid
    \objviol{\Delta,d}=\objopt{\Delta}\}, \\
    \subjopta{\Delta}
     & =
    \{d\in\decisions{\Delta}\mid
    \subjviol{\Delta,d}=\subjopt{\Delta}\}.
  \end{align*}
\end{defn}

For a DNCOP $\Delta=\dncoptuple$ and every $d\subseteq\decatoms$, the objective description
determines a propositional valuation $\objval{\Delta,d}=\objdesc{\Delta}{d}\cap\statevocab$, since
it contains exactly one of $\alpha$ and $\neg\alpha$ for every $\alpha\in\statevocab$. All
objectively compatible models assign the same truth values to the atoms in $\statevocab$. Since
each $\violformula{n}$ is a Boolean combination of these atoms, it can be evaluated under
$\objval{\Delta,d}$. Thus, whenever $d$ is objectively admissible, the maximization in the
definition of $\objviol{\Delta,d}$ is vacuous and the violation count can be computed directly from
$\objval{\Delta,d}$.

\begin{prop}\label{thm:objective-determinacy}
  Let $\Delta=\dncoptuple$ be a DNCOP and let $d\subseteq\decatoms$ be objectively admissible. For
  every pointed model $(\M,w)\in\objcomp{\Delta}{d}$, we have $\violcount{\M}{w} =
    \sum_{\substack{n\in\countednorms\\ \objval{\Delta,d}\modelspl\violformula{n}}}\normweight(n)$.
  Consequently, $\objviol{\Delta,d} = \sum_{\substack{n\in\countednorms\\
        \objval{\Delta,d}\modelspl\violformula{n}}}\normweight(n)$.
\end{prop}

The maximization over compatible models is still relevant from the subjective perspective, since
atoms relevant to the computation of violation counts may be open under the beliefs of the acting
agent.

\begin{ex}\label{ex:vehicle-factual}
  Consider the DNCOPs $\Delta_{\neg c}$ and $\Delta_c$ from
  Ex.~\ref{ex:vehicle-factual-specification}. Both DNCOPs are decision-sound and in both variants,
  the agent believes $e$ and $\O(r|e)$ -- so they always believe the obligation is in force. Since
  $c$ is left open in the agent's beliefs, moving right risks the cost-2 violation in the worst case,
  while not moving right always costs 1 (violating $\O(r|e)$); not moving right is therefore
  subjectively optimal regardless of the actual truth of $c$. Objectively, the optimal decision
  matches the true value of $c$: if no cyclist is present, moving right is objectively optimal; if
  one is present, not moving right is objectively optimal, as it is subjectively.
\end{ex}

\subsection{When the Two Perspectives Coincide}
\label{sec:subjobj}

The mismatch described in the above example (subjective and objective optimum disagreeing) arises
because the norms' conditions ($e$ and $c$) are objective facts the agent may be unsure about. If
instead every norm and every norm's condition is stated directly in terms of the agent's own
beliefs, its truth value is automatically settled by the agent's beliefs at every subjectively
compatible model, making the subjective violation cost determinate. If these beliefs are also
truthful, the two perspectives coincide.
\begin{defn}
  A DNCOP $\Delta=\dncoptuple$ is \emph{local} if every norm $n\in\countednorms$, writing $\omega$
  for its content and $\phi$ for its condition (in the case of explicit norms), satisfies the
  following conditions:
  \begin{enumerate}[label=(\roman*)]
    \item either $n\in\decatoms$, or exactly one of $B_{\actagent}n$ and $B_{\actagent}\neg n$ belongs to
          $\trueatoms$;
    \item either $\omega\in\decatoms$, or exactly one of $B_{\actagent}\omega$ and $B_{\actagent}\neg\omega$
          belongs to $\trueatoms$; and
    \item $\phi\in\mathsf{Bool}(\siga)$ (if $n$ is an explicit norm).
  \end{enumerate}
  A DNCOP is \emph{truthful} if every belief of the acting agent about a fixed atom is correct.
  Formally, for every $\alpha\in\trueatoms\cup\falseatoms$, we require
  $B_{\actagent}\alpha\in\trueatoms\Longrightarrow\alpha\in\trueatoms$ and
  $B_{\actagent}\neg\alpha\in\trueatoms\Longrightarrow\alpha\in\falseatoms$.
\end{defn}
Locality reflects a common design pattern for actionable norms: rather than conditioning
obligations on facts the agent may only partially observe, the norm is phrased directly in terms of
the agent's own (introspectively accessible) beliefs. The next example restates the norms of
Example~\ref{ex:vehicle-factual} in this style, and we will see that this removes the discrepancy
between the two examples. Truthfulness is a complementary idealization, assuming the agent's
beliefs are accurate.

\begin{ex}\label{ex:vehicle-doxastic}
  Replace in Ex.~\ref{ex:vehicle-factual-specification} the norms with $n_O=\O(r | B_{\actagent}
    e\land\neg\hat{B}_{\actagent} c)$ (``move right if you believe an emergency vehicle is coming and
  do not consider a cyclist possible'') and $n_F=\F(r | \hat{B}_{\actagent} c)$ (``do not move right
  if you cannot rule out a cyclist'') with $\normweight(n_O)=\normweight(n_F)=1$. Hence
  $\countednorms=\{n_O\COMMA n_F\}$, $\decatoms=\{r\}$, and
  $$\trueatoms =
    \{e\COMMA n_O\COMMA n_F\COMMA B_{\actagent}e\COMMA B_{\actagent}n_O\COMMA B_{\actagent}n_F\},
    \quad
    \falseatoms=
    \{c\COMMA B_{\actagent}c\COMMA B_{\actagent}\neg c\}.
  $$
  The resulting DNCOP is decision-sound. Objectively, an emergency vehicle is approaching from behind
  and no cyclist is present to the right of the car. The acting agent believes that the emergency
  vehicle is approaching but has no belief about whether a cyclist is present. Since
  $B_{\actagent}\neg c \in \falseatoms$, the formula $\hat{B}_{\actagent} c$ is true: moving right
  has objective and subjective cost one, whereas not moving right has cost zero. Thus, not moving
  right is uniquely optimal from both perspectives. Both norms condition only on the agent's own
  beliefs, so the DNCOP is local; since the agent's beliefs are in fact accurate, it is truthful too.
  The two conditions are mutually exclusive, so no priority between the norms is ever needed.

\end{ex}
We show below that in a local, decision-sound DNCOP $\Delta$, all subjectively compatible models
for a given decision agree on the violation count: the worst case in $\subjviol{\Delta,d}$ is not
actually in doubt.
\begin{prop}\label{prop:subjective-determinacy}
  Let $\Delta$ be a decision-sound local DNCOP. For every $d\in\decisions{\Delta}$ and all
  $(\M,w),(\M',w')\in\subjcomp{\Delta}{d}$, $\violcount{\M}{w}=\violcount{\M'}{w'}$. Hence the
  maximization defining $\subjviol{\Delta,d}$ is vacuous.
\end{prop}

\begin{proof}
  For each $n\in\countednorms$, the choice or, by reflexivity, the corresponding belief fixes both
  $n$ and its content in every subjectively compatible model. The subjective fixed and choice
  descriptions fix all RDDAs in its local condition. Hence all such models agree on whether each $n$
  is violated, and therefore on the total count.
\end{proof}

If the DNCOP is additionally truthful, the subjective and objective violation counts coincide for
every decision, and hence so do the optimal decisions and their violation counts under both
perspectives.
\begin{prop}\label{prop:subjective-objective-agreement}
  Let $\Delta$ be a decision-sound, local, and truthful DNCOP. Then
  $\subjviol{\Delta,d}=\objviol{\Delta,d}$ for every $d\in\decisions{\Delta}$. Consequently,
  $\subjopt{\Delta}=\objopt{\Delta}$ and $\subjopta{\Delta}=\objopta{\Delta}$.
\end{prop}

\begin{proof}
  For decision atoms, the objective and subjective descriptions agree by construction. For the other
  atoms, locality provides a positive or negative belief, and truthfulness makes $a$'s beliefs agree
  with what is actually true. Both descriptions also agree on every RDDA in norm conditions. Thus
  every subjectively compatible model evaluates each $\violformula{n}$ according to
  $\objval{\Delta,d}$. Prop.~\ref{thm:objective-determinacy} gives equality of the costs, and both
  optimizations range over $\decisions{\Delta}$.
\end{proof}

\subsection{Reduction to MaxSAT}\label{sec:maxsat}

For local DNCOPs, every subjective violation is a Boolean function of the decision. We can
therefore optimize subjective violations directly with MaxSAT. Under truthfulness, the same
construction also optimizes objective violations.

\begin{defn}\label{def:maxsat-translation}
  Let $\Delta=\dncoptuple$ be a local DNCOP. We directly use the decision atoms $\delta\in\decatoms$
  as Boolean variables for the MaxSAT problem. Apart from these decision atoms, the translation uses
  only the acting agent's belief part of the DNCOP, given by $\trueatoms\cap\siga$ and
  $\falseatoms\cap\siga$.

  We first translate the decision constraint $\chi$ polynomially to a CNF using a Tseitin
  transformation~\cite{Tseitin68}. The resulting clauses are used as hard constraints.

  For every $n\in N$, we then define a Boolean formula $\widehat{\chi}_n$ over $\decatoms$
  characterizing the condition $\chi_n$ under which $n$ is violated. We obtain this formula by
  replacing each RDDA occurring directly in $\chi_n$, that is, each RDDA occurring as a propositional
  atom of $\chi_n$, as follows:
  \begin{enumerate}[label=(\roman*)]
    \item any $\delta$ or $B_a\delta$ with $\delta\in\decatoms$ is replaced by $\delta$,
    \item any $B_a\neg\delta$ with $\delta\in\decatoms$ is replaced by $\neg\delta$,
    \item any remaining objective atom $\alpha\in\atomsb$ (such as $n$ itself and its content $\omega$) is
          replaced by $\top$ if $B_a\alpha\in\trueatoms\cap\siga$, and by $\bot$ if
          $B_a\neg\alpha\in\trueatoms\cap\siga$, and
    \item any remaining $\alpha\in\siga$ is replaced by $\top$ if $\alpha\in\trueatoms\cap\siga$, and by
          $\bot$ if $\alpha\in\falseatoms\cap\siga$.
  \end{enumerate}
  Locality guarantees that every RDDA occurring in $\chi_n$ is covered by exactly one of these cases.
  For each norm $n\in N$, we then introduce a fresh variable $z_n$ and impose
  $z_n\leftrightarrow\widehat{\chi}_n$ as a hard constraint, adding $\neg z_n$ as a soft clause of
  weight $v(n)$. Thus, weight $v(n)$ is incurred exactly when $n$ is violated. All hard constraints
  are converted to CNF using a Tseitin transformation.

\end{defn}

\begin{ex}
  Consider the local DNCOP from Example~\ref{ex:vehicle-doxastic}. Its violation conditions are
  $\chi_{n_O}=n_O\land B_a e\land\neg\hat B_a c\land\neg r$ and $\chi_{n_F}=n_F\land\hat B_a c\land
    r$.

  The fixed beliefs make $n_O,n_F$, and $B_a e$ true, while $B_a\neg c$ is false and hence $\hat B_a
    c=\neg B_a\neg c$ is true. Therefore, the translation gives $\widehat{\chi}_{n_O}=\bot$ and
  $\widehat{\chi}_{n_F}=r$. Since the decision constraint is $\top$, the hard clauses simplify to $
    \neg z_{n_O}, \neg z_{n_F}\lor r$, and $z_{n_F}\lor\neg r, $ and the soft clauses are $(\neg
    z_{n_O},1)$ and $(\neg z_{n_F},1)$. Setting $r$ false satisfies both soft clauses, while setting it
  true falsifies the second at cost one. Thus the unique optimal assignment encodes $d=\emptyset$, as
  found in Ex.~\ref{ex:vehicle-doxastic}.
\end{ex}

\begin{thm}\label{thm:maxsat-optimization}
  For a decision-sound local DNCOP $\Delta$, the construction above is a polynomial-time reduction to
  weighted partial MaxSAT. The restrictions of its optimal assignments to $\decatoms$ encode exactly
  the subjectively optimal decisions. If $\Delta$ is also truthful, they also encode exactly the
  objectively optimal decisions.
\end{thm}

\begin{proof}
  For $d\subseteq\decatoms$, let $I_d$ be the valuation defined by $I_d(\delta)=1$ iff $\delta\in d$.
  By correctness of the Tseitin transformation, $I_d$ can be extended to a valuation satisfying the
  hard clauses obtained from $\chi$ iff $d\modelspl\chi$, that is, iff $d\in\decisions{\Delta}$.

  Now consider any $n\in N$. Cases (i) and (ii) reproduce the values fixed by the decision, case
  (iii) reproduces the acting agent's fixed beliefs about the remaining objective atoms, and case
  (iv) reproduces the fixed values of the remaining belief atoms. By locality, these cases cover
  every RDDA occurring directly in $\chi_n$. A structural induction on $\chi_n$ therefore shows that
  $I_d\models\widehat{\chi}_n$ iff $n$ is subjectively violated under $d$. The hard constraint
  $z_n\leftrightarrow\widehat{\chi}_n$ thus makes $z_n$ true exactly when $n$ is subjectively
  violated. Consequently, the total weight of the falsified soft clauses in any hard-satisfying
  extension of $I_d$ equals $\subjviol{\Delta,d}$ by Proposition~\ref{prop:subjective-determinacy}.
  Conversely, every assignment satisfying the hard constraints restricts to $I_d$ for some
  $d\in\decisions{\Delta}$. Minimizing the weight of the falsified soft clauses therefore yields
  exactly the subjectively optimal decisions.

  If $\Delta$ is truthful, Proposition~\ref{prop:subjective-objective-agreement} gives
  $\subjviol{\Delta,d}=\objviol{\Delta,d}$ for every $d\in\decisions{\Delta}$. Hence the subjective
  and objective optimal decisions coincide. Finally, the replacements increase the formula size at
  most linearly, and the Tseitin transformations are polynomial.
\end{proof}

\section{Conclusions and Future Work}
We introduced a simple logical framework for norm-guided decision making under uncertainty.
Starting from a restricted fragment of Chellas' Minimal Deontic Logic, we identified a language
where deontic reasoning reduces to classical propositional reasoning while retaining obligations
and prohibitions over atomic actions. We extended this fragment with explicitly represented
conditional norms and factual detachment, obtaining a lightweight deontic logic.

To reason under incomplete information, we combined the framework with multi-agent doxastic logic,
allowing norms to depend on agents' beliefs and beliefs about norms. We introduced the Doxastic
Norm Compliance Optimization Problem (DNCOP), where an agent selects decisions minimizing weighted
norm violations. By distinguishing optimization relative to the agent's beliefs from optimization
relative to the objective state of the world, the framework captures a key aspect of autonomous
decision making under uncertainty. We identified locality and truthfulness conditions under which
these perspectives coincide, guaranteeing that subjectively optimal decisions are also objectively
optimal.

The framework also admits efficient computation. For local, decision-sound DNCOPs, we showed that
optimal decisions can be found by a polynomial-time reduction to weighted partial MaxSAT,
leveraging mature optimization techniques while preserving logical rigour and transparency.

The framework deliberately favors simplicity over expressiveness. Restricting normative contents to
atomic propositions enables reductions to propositional reasoning and straightforward integration
with epistemic formalisms, while still covering a broad range of practical decision-making
scenarios. Its modular design leaves room for richer normative and epistemic extensions.

Several directions remain for future work. We plan to analyze LDLB in more detail; we conjecture
that LDLB-satisfiability is PSPACE-complete, as is plain multi-agent
KD45~\cite{DBLP:journals/ai/HalpernM92}. An immediate objective is to extend the optimization
framework beyond local DNCOPs, where subjective optimization may require reasoning over multiple
compatible belief models. We also plan to enrich the representation of agents' beliefs, allowing
belief bases to capture more complex propositional information, including constitutive norms.
Appropriate candidates are belief bases in the style of Lorini~\cite{DBLP:journals/ai/Lorini20}.
Another promising direction is to study decentralized settings in which multiple agents make
decisions simultaneously while reasoning about one another's beliefs and intended actions. Finally,
we intend to implement the MaxSAT translation and experimentally evaluate its scalability on
representative normative decision-making benchmarks.

\begingroup
\small
\subsubsection*{Acknowledgments}
This research was funded in whole or in part by the Vienna Science and Technology Fund (WWTF)
project ICT22-023 and the Austrian Science Fund (FWF) 10.55776/COE12 and 10.55776/PAT1815025.
\endgroup

\bibliographystyle{splncs04}
\bibliography{bibliography}

@book{Hintikka1962,
author = {Hintikka, Jaakko},
publisher = {Cornell University Press},
title = {{Knowledge and Belief: An Introduction to the Logic of the Two Notions}},
year = {1962}
}

@inproceedings{DBLP:conf/deon/CiabattoniU025,
  author       = {Agata Ciabattoni and
                  Blaz Istenic Urh and
                  Xavier Parent},
  title        = {Deontic Explanations in {\AA}qvist's Systems},
  booktitle    = {{DEON}},
  pages        = {79--98},
  publisher    = {College Publications},
  year         = {2025}
}

@inproceedings{DBLP:conf/ijcai/NeufeldCT25,
  author       = {Emery A. Neufeld and
                  Agata Ciabattoni and
                  Radu Florin Tulcan},
  title        = {Combining {MORL} with Restraining Bolts to Learn Normative Behaviour},
  booktitle    = {{IJCAI}},
  pages        = {4615--4623},
  publisher    = {ijcai.org},
  year         = {2025}
}

@Incollection{A84,
author="{\AA}qvist, L.",
editor="Gabbay, D. and Guenthner, F.",
title="Deontic Logic",
bookTitle="Handbook of Philosophical Logic: Volume II",
year="1984",
publisher="Springer",
pages="605--714"
}

@book{mill1863utilitarianism,
  author     = {Mill, John Stuart},
  title      = {Utilitarianism},
  publisher = {Parker, Son, and Bourn},
  year      = {1863}
}

@article{DBLP:journals/cacm/SergotSKKHC86,
  author       = {Marek J. Sergot and
                  Fariba Sadri and
                  Robert A. Kowalski and
                  F. Kriwaczek and
                  Peter Hammond and
                  H. T. Cory},
  title        = {The British Nationality Act as a Logic Program},
  journal      = {Commun. {ACM}},
  volume       = {29},
  number       = {5},
  pages        = {370--386},
  year         = {1986}
}

@article{DBLP:journals/jphil/KowalskiS18,
  author       = {Robert A. Kowalski and
                  Ken Satoh},
  title        = {Obligation as Optimal Goal Satisfaction},
  journal      = {J. Philos. Log.},
  volume       = {47},
  number       = {4},
  pages        = {579--609},
  year         = {2018}
}

@article{Tseitin68,
  author       = {G. S. Tseitin},
  title        = {On the Complexity of Derivation in Propositional Calculus},
  journal      = {Zap. Nauchn. Sem. LOMI},
  volume       = {8},
  pages        = {234--259},
  year         = {1968}
}

@book{DBLP:books/cu/Chellas80,
  author       = {Brian F. Chellas},
  title        = {Modal Logic - An Introduction},
  publisher    = {Cambridge University Press},
  year         = {1980}
}

@article{DBLP:journals/jphil/MakinsonT00,
  author       = {David Makinson and
                  Leendert W. N. van der Torre},
  title        = {Input/Output Logics},
  journal      = {J. Philos. Log.},
  volume       = {29},
  number       = {4},
  pages        = {383--408},
  year         = {2000}
}

@inproceedings{DBLP:conf/jelia/CabalarCT23,
  author       = {Pedro Cabalar and
                  Agata Ciabattoni and
                  Leendert van der Torre},
  title        = {Deontic Equilibrium Logic with eXplicit Negation},
  booktitle    = {{JELIA}},
  series       = {Lecture Notes in Computer Science},
  volume       = {14281},
  pages        = {498--514},
  publisher    = {Springer},
  year         = {2023}
}

@article{von1967deontic,
  author       = {G. H. von Wright},
  title        = {Deontic Logics},
  journal      = {Am. Philos. Q.},
  volume       = {4},
  number       = {2},
  pages        = {136--143},
  year         = {1967}
}

@article{Ross1941,
  author       = {Alf Ross},
  title        = {Imperatives and Logic},
  journal      = {Philos. Sci.},
  volume       = {11},
  number       = {1},
  pages        = {30--46},
  year         = {1944}
}

@article{DBLP:journals/jphil/GovernatoriORS13,
  author       = {Guido Governatori and
                  Francesco Olivieri and
                  Antonino Rotolo and
                  Simone Scannapieco},
  title        = {Computing Strong and Weak Permissions in Defeasible Logic},
  journal      = {J. Philos. Log.},
  volume       = {42},
  number       = {6},
  pages        = {799--829},
  year         = {2013}
}

@Book{Hand13,
  editor =	 {Dov Gabbay and John Horty and Xavier Parent and Ron van der Meyden and Leendert van der Torre},
  title = 	 {Handbook of Deontic Logic and Normative Systems},
  publisher = 	 {College Publications},
  year = 	 {2013},
}

@book{Hand21,
    title = {Handbook of Deontic Logic and Normative Systems, Volume 2},
    editor = {Dov Gabbay and John Horty and Xavier Parent and Ron van der Meyden and Leendert van der Torre},
    year = {2021},
  publisher = 	 {College Publications},
}

@incollection{xavier_survey,
    author = {Xavier Parent},
  title = {Preference Semantics for {Hansson}-type Dyadic Deontic Logic: A Survey of Results},
  booktitle = {Handbook of Deontic Logic and Normative Systems, Volume 2},
    publisher = {College Publications},
    year = {2021},
pages = {7--69},
editor ={D. Gabbay and J. Horty and X. Parent and R. van der Meyden and L. van der Torre}
}

@inproceedings{DBLP:conf/icail/GiordanoMD13,
  author       = {Laura Giordano and
                  Alberto Martelli and
                  Daniele Theseider Dupr{\'{e}}},
  title        = {Temporal deontic action logic for the verification of compliance to
                  norms in {ASP}},
  booktitle    = {{ICAIL}},
  pages        = {53--62},
  publisher    = {{ACM}},
  year         = {2013}
}

@article{DBLP:journals/jancl/AlferesGL13,
  author       = {Jos{\'{e}} J{\'{u}}lio Alferes and
                  Ricardo Gon{\c{c}}alves and
                  Jo{\~{a}}o Leite},
  title        = {Equivalence of defeasible normative systems},
  journal      = {J. Appl. Non Class. Logics},
  volume       = {23},
  number       = {1-2},
  pages        = {25--48},
  year         = {2013}
}

@article{DBLP:journals/jphil/MakinsonT01,
  author       = {David Makinson and
                  Leendert W. N. van der Torre},
  title        = {Constraints for Input/Output Logics},
  journal      = {J. Philos. Log.},
  volume       = {30},
  number       = {2},
  pages        = {155--185},
  year         = {2001}
}

@article{DBLP:journals/jolli/BerkelCFGO23,
  author       = {Kees van Berkel and
                  Agata Ciabattoni and
                  Elisa Freschi and
                  Francesca Gulisano and
                  Maya Olszewski},
  title        = {Deontic Paradoxes in M{\={\i}}m{\={a}}{\d{m}}s{\={a}} Logics:
                  There and Back Again},
  journal      = {J. Log. Lang. Inf.},
  volume       = {32},
  number       = {1},
  pages        = {19--62},
  year         = {2023}
}

@inproceedings{DBLP:conf/ecai/LakemeyerL12,
  author       = {Gerhard Lakemeyer and
                  Yves Lesp{\'{e}}rance},
  title        = {Efficient Reasoning in Multiagent Epistemic Logics},
  booktitle    = {{ECAI}},
  series       = {Frontiers in Artificial Intelligence and Applications},
  volume       = {242},
  pages        = {498--503},
  publisher    = {{IOS} Press},
  year         = {2012}
}

@inproceedings{DBLP:conf/aaai/MuiseBFMMPS15,
  author       = {Christian J. Muise and
                  Vaishak Belle and
                  Paolo Felli and
                  Sheila A. McIlraith and
                  Tim Miller and
                  Adrian R. Pearce and
                  Liz Sonenberg},
  title        = {Planning Over Multi-Agent Epistemic States: {A} Classical Planning
                  Approach},
  booktitle    = {{AAAI}},
  pages        = {3327--3334},
  publisher    = {{AAAI} Press},
  year         = {2015}
}

@article{DBLP:journals/ai/HalpernM92,
  author       = {Joseph Y. Halpern and
                  Yoram Moses},
  title        = {A Guide to Completeness and Complexity for Modal Logics of Knowledge
                  and Belief},
  journal      = {Artif. Intell.},
  volume       = {54},
  number       = {3},
  pages        = {319--379},
  year         = {1992}
}

@article{DBLP:journals/ai/Lorini20,
  author       = {Emiliano Lorini},
  title        = {Rethinking epistemic logic with belief bases},
  journal      = {Artif. Intell.},
  volume       = {282},
  pages        = {103233},
  year         = {2020}
}

@inproceedings{DBLP:conf/sat/AnsoteguiBL09,
  author       = {Carlos Ans{\'{o}}tegui and
                  Maria Luisa Bonet and
                  Jordi Levy},
  title        = {Solving (Weighted) Partial MaxSAT through Satisfiability Testing},
  booktitle    = {{SAT}},
  series       = {Lecture Notes in Computer Science},
  volume       = {5584},
  pages        = {427--440},
  publisher    = {Springer},
  year         = {2009}
}

\end{document}